\documentclass[english,a4paper,12pt]{article}

\usepackage{amsxtra, amsfonts, amssymb, amstext, amsmath, mathtools}
\usepackage{amsthm}
\usepackage{booktabs}
\usepackage{nicefrac}
\usepackage{xspace}
\usepackage{url}
\usepackage{graphics,color}
\usepackage[ruled,vlined,linesnumbered]{algorithm2e}
\usepackage{wrapfig}
\usepackage{lmodern}
\usepackage{dsfont}
\usepackage[shortcuts]{extdash}

\newcommand{\oea}{\mbox{${(1 + 1)}$~EA}\xspace}

\newcommand{\rUMDA}{$r$\=/UMDA\xspace} 

\newcommand{\onemax}{\textsc{OneMax}\xspace}
\newcommand{\LO}{\textsc{Leading\-Ones}\xspace}
\newcommand{\DLB}{\textsc{Deceptive\-LeadingBlocks}\xspace}
\newcommand{\leadingones}{\LO}

\newcommand{\rLO}{$r$\=/\LO}

\newcommand*{\Tsel}{T^{\textrm{sel}}}

\newcommand{\R}{\ensuremath{\mathbb{R}}}

\newcommand{\N}{\ensuremath{\mathbb{N}}} 

\newcommand{\bbone}{{\mathds{1}}}

\newcommand{\calF}{\ensuremath{\mathcal{F}}}

\newcommand{\Var}{\mathrm{Var}\xspace} 
\newcommand{\eps}{\varepsilon}

\let\originalleft\left
\let\originalright\right
\renewcommand{\left}{\mathopen{}\mathclose\bgroup\originalleft}
\renewcommand{\right}{\aftergroup\egroup\originalright}

\usepackage{bbm}
\usepackage[utf8]{inputenc}
\usepackage{hyperref}
\usepackage{cleveref}
\DeclareMathOperator{\E}{\mathds{E}}

\newtheorem{theorem}{Theorem}
\newtheorem{lemma}[theorem]{Lemma}

\title{Estimation-of-Distribution Algorithms for Multi-Valued Decision Variables}

\author{Firas Ben Jedidia$^{1}$ \and Benjamin Doerr$^{2}$ \and Martin S. Krejca$^{2}$}

\date
{
    \small $^{1}$École Polytechnique, Institut Polytechnique de Paris, Palaiseau, France\\
    \small $^{2}$Laboratoire d'Informatique (LIX), CNRS, École Polytechnique,\\
    \small Institut Polytechnique de Paris, Palaiseau, France
}

\begin{document}
{\sloppy
\maketitle

\begin{abstract}
    The majority of research on estimation-of-distribution algorithms (EDAs) concentrates on pseudo-Boolean optimization and permutation problems, leaving the domain of EDAs for problems in which the decision variables can take more than two values, but which are not permutation problems, mostly unexplored.
    To render this domain more accessible, we propose a natural way to extend the known univariate EDAs to this setting.
    Different from a na\"ive reduction to the binary case, our approach avoids additional constraints.

    Since understanding genetic drift is crucial for an optimal parameter choice, we extend the known quantitative analysis of genetic drift to EDAs for multi-valued variables. Roughly speaking, when the variables take $r$ different values, the time for genetic drift to become significant is $r$ times shorter than in the binary case. Consequently, the update strength of the probabilistic model has to be chosen $r$ times lower now.

    To investigate how desired model updates take place in this framework, we undertake a mathematical runtime analysis on the $r$-valued \leadingones problem. We prove that with the right parameters, the multi-valued UMDA solves this problem efficiently in $O(r\ln(r)^2 n^2 \ln(n))$ function evaluations. This bound is nearly tight as our lower bound $\Omega(r\ln(r) n^2 \ln(n))$ shows.

    Overall, our work shows that our good understanding of binary EDAs naturally extends to the multi-valued setting, and it gives advice on how to set the main parameters of multi-values EDAs.
\end{abstract}


\section{Introduction}

Estimation-of-distribution algorithms (EDAs~\cite{PelikanHL15}) are randomized search heuristics that evolve a probabilistic model of the search space (that is, a probability distribution over the search space). In contrast to solution-based algorithms such as classic evolutionary algorithms, which only have the choice between the two extreme decisions of keeping or discarding a solution, EDAs can take into account the information gained from a function evaluation also to a smaller degree. This less short-sighted way of reacting to new insights leads to several proven advantages, e.g., that EDAs can be very robust to noise~\cite{FriedrichKKS17,LehreN19gecco}. Since the evolved distributions often have a larger variance, EDAs can also be faster in exploring the search space, in particular, when it comes to leaving local optima, where they have been shown to significantly outperform simple evolutionary algorithms~\cite{HasenohrlS18,Doerr21cgajump,WangZD21,BenbakiBD21,DoerrK21ecj,Witt23}.

While EDAs have been employed in a variety of settings and to different types of decision variables~\cite{PelikanHL15,LarranagaL02}, the number of results in which they have been used for discrete optimization problems with decision variables taking more than two values, other than permutation problems, is scarce~\cite{SantanaLL08ProteinFolding,SantanaLL10FactorizationEDAs,SantanaORS02IntegerEDAs,SantanaM13MarkovNetworkEDAs,Muehlenbein97}.
All of these results have in common that they propose specific EDAs to deal with multi-valued problems.
To the best of our knowledge, no systematic way to model EDAs for the multi-valued domain exists, even not for the easiest case of EDAs that do not model dependencies, so-called \emph{univariate} EDAs (we note that multi-variate EDAs are much less understood, e.g., despite some theoretical works in this direction~\cite{LehreN19foga,DoerrK23tcs}, there are no proven runtime guarantees for these algorithms).

Since this might be a lost opportunity, we undertake the first steps towards a framework of univariate EDAs for problems with decision variables taking more than two values (but different from permutation problems).
We first note that the strong dependencies that distinguish a permutation problem from just a problem defined on $\{1, \dots, n\}^n$ have led to very particular EDAs for permutation problems. We did not see how to gain insights from these results for general multi-valued problems.

We therefore define EDAs for multi-valued decision variables from scratch, that is, without building on any related existing work. We note that, in principle, one could transform a multi-valued problem into a binary one by having, for each variable taking $r$ different values, $r$ binary variables, each indicating that the variable has the corresponding value. This would lead to a constrained optimization problem with the additional constraints that exactly one of these variables can take the value~$1$. This might be a feasible approach, but since such constraints generally impose additional difficulties, we propose a way that does not need an additional treatment of constraints (in other words, we set up our EDAs in a way that these constraints are satisfied automatically).

We defer the details to \Cref{sec:framework} and only sketch the rough idea of our approach here. For each variable taking $r$ values, without loss of generality the values $\{0, \dots, r-1\}$, we have $r$ sampling frequencies $p_0, p_1, \dots, p_{r-1}$ that always add up to~$1$. When sampling a value for the variable, we do this mutually exclusively, that is, the variable takes the value $i$ with probability~$p_i$. This mutual exclusion in the sampling immediately gives that the frequency update does not violate the property that the frequencies add up to~$1$. Consequently, this appears to be a convenient (and in fact very natural) set-up for a multi-valued~EDA. We note that there are some non-trivial technical questions to be discussed when working with frequencies borders, such as $ \left[ \frac 1n,1-\frac 1n\right]$ in the classical binary case, but we also come up with a simple and natural solution for this aspect.

As a first step towards understanding this multi-valued EDA framework, we study how prone it is to genetic drift. Genetic drift in EDAs means that sampling frequencies not only move because of a clear signal induced by the objective function, but also due random fluctuations in the sampling process. This has the negative effect that even in the complete absence of a fitness signal, the EDA develops a preference for a particular value of this decision variable. From a long sequence of works, see \Cref{sec:geneticDrift} for the details, it is well understood how the time for this genetic-drift effect to become relevant depends on the parameters of the EDA~\cite{DoerrZ20tec}. Consequently, if one plans to run the EDA for a certain number of iterations, then this quantification tells the user how to set the parameters as to avoid genetic drift within this time period.

Since such a quantification is apparently helpful in the application of EDAs, we first extend this quantification to multi-valued EDAs. When looking at the relatively general tools used in~\cite{DoerrZ20tec}, this appears straightforward, but it turns out that such a direct approach does not give the best possible result. The reason is that for multi-valued decision variables, the martingale describing a frequency of a neutral variable over time has a lower variance (in the relevant initial time interval). To profit from this, we use a fairly technical martingale concentration result of McDiarmid~\cite{McDiarmid98}, which, to the best our our knowledge, has not been used before in the analysis of randomized search heuristics. Thanks to this result, we show that the time for genetic drift to become relevant is (only) by a factor of $r$ lower than in the case of binary decision variables (\Cref{th:neutral_bound}).

We use this result to conduct a mathematical runtime analysis of the multi-valued univariate marginal distribution algorithm (\rUMDA) on the $r$-valued \leadingones problem in the regime with low genetic drift. This problem is interesting since a typical optimization process optimizes the variable sequentially in a fixed order. Consequently, in a run of an EDA on \leadingones, there is typically always one variable with undecided sampling frequency that has a strong influence on the fitness. Hence, this problem is suitable to study how fast an EDA reacts to a strong fitness signal.

Our runtime analysis shows that also in the multi-valued setting, EDAs can react fast to a strong fitness signal. Since now the frequencies start at the value~$\frac 1r$, the time to move a frequency is a little longer, namely $\Theta(r \ln(r))$ instead of constant when the sample size~$\lambda$ is by a sufficient constant factor larger than the selection size~$\mu$. This still appears to be a small price for having to deal with~$r$ decision alternatives. This larger time also requires that the model update has to be chosen more conservatively as to prevent genetic drift (for this, we profit from our analysis of genetic drift), leading to another $\ln(r)$ factor in the runtime. In summary, we prove (\Cref{thm:rUMDAonrLO}) that the UMDA can optimize the $r$-valued \leadingones problem in time $O(r\ln(r)^2 n^2 \ln(n))$, a bound that agrees with the one shown in~\cite{DoerrK21tcs} for the classical case $r=2$. Our upper bound is tight apart from a factor logarithmic in $r$, that is, we prove a lower bound of order $\Omega(r\ln(r) n^2 \ln(n))$ in Theorem~\ref{thm:rUMDALowerBound}.

Overall, our work shows that $r$-valued EDAs can be effective problem solvers, suggesting to apply such EDAs more in practice.

This work extends our prior extended abstract~\cite{BenJedidiaDK23} by adding a lower bound for the runtime of the $r$-valued UMDA on the $r$-valued \leadingones problem. Also, it contains all proofs that were omitted in the conference version for reasons of space. To avoid misunderstandings, we note that this work bears no similarity or overlap with the paper \emph{Generalized Univariate Estimation-of-Distribution Algorithms}~\cite{DoerrD22}, which studies generalized update mechanisms for EDAs for binary decision variables.

This article is organized as follows. We describe previous works in the following section and set the notation in the subsequent section. In \Cref{sec:multiValuedEDAs}, we propose our multi-valued EDA framework. Our main technical results, the analysis of genetic drift and the runtime analysis for the \leadingones problem, can be found in \Cref{sec:geneticDrift,sec:runtimeAnalysis}. The paper ends with a short conclusion.

\section{Related Work}\label{sec:related}

Since the technical sections of this work contain three relatively independent topics---the definition of multi-valued EDAs, genetic drift, and a runtime analysis on the \leadingones benchmark---we present the previous works relevant to these topics in the respective sections. We hope that this eases the reading of this paper.

This being a theoretical work, we do not discuss in detail how EDAs have been successfully used to solve real-worlds optimization problems and refer to the surveys~\cite{LarranagaL02,PelikanHL15}.

Theoretically oriented works have accompanied the development and use of EDAs for a long time, see, e.g., the early works on genetic drift described in \Cref{sec:geneticDrift}. The first mathematical runtime analysis of an EDA was conducted by Droste~\cite{Droste06}. This seminal work, showing an asymptotically tight bound for the runtime of the compact genetic algorithm on the \onemax benchmark, already contains many ideas that are now frequently used in the runtime analysis of EDAs. It also observed that EDAs optimize problems in a very different manner, visible from the different runtimes shown on two linear functions, which contrasts the famous analysis of how the \oea optimizes linear functions by Drose, Jansen, and Wegener~\cite{DrosteJW02}. Interestingly, apart from the works of one research group~\cite{ChenTCY07,ChenLTY09,ChenTCY10}, Droste's ground-breaking work~\cite{Droste06} was not followed up by other runtime analyses for around ten years. Since then, starting with works like~\cite{DangL15,FriedrichKK16,SudholtW16,KrejcaW17}, the runtime analysis of EDAs has become very active and has, despite the technical challenges in analyzing such complex algorithms, produced many fundamental results and a good understanding of some of the working principles of EDAs. We refer to the recent survey~\cite{KrejcaW20bookchapter} for more details.

\section{Preliminaries}
\label{sec:preliminaries}
We denote by~$\N$ the set of all natural numbers, including~$0$, and by~$\R$ the set of all real numbers.
Additionally, for $a, b\in \N$, let $\left[a..b\right]=[a,b]\cap \N$, and let $[a] = [1 .. a]$.
When we say that a random process is a martingale and do not specify a filtration, then we mean that the process is a martingale with respect to its natural filtration.
Further, for all $n \in \N_{\geq 1}$ and $p \in \R_{\geq 0}^n$, we denote the $1$-norm of~$p$, that is, the sum of the entries of~$p$, by~$\|p\|_1$.

Let $n\in \N_{\geq 1}$ and $r\in \N_{\geq 2}$.
We consider the maximization of functions of the form $f\colon \left[0..r-1\right]^{n} \rightarrow \R$, which we call \emph{r-valued fitness functions}.
Whenever we mention an $r$-valued fitness function, we implicitly assume that its dimension~$n$ and the cardinality~$r$ of its domain are given.
We call each $x\in \left[0..r-1\right]^{n}$ an \emph{individual}, and we call $f(x)$ the \emph{fitness} of $x$.

We say that a random variable~$Y$ \emph{stochastically dominates} another random variable~$X$, not necessarily defined on the same probability space, denoted by $ X \preceq Y$, if and only if for all $\lambda \in \R$, we have $\Pr\left[ X\leq \lambda \right] \leq \Pr \left[ Y\leq \lambda \right]$.

\section{Multi-Valued EDAs}\label{sec:multiValuedEDAs}

In this section, we generalize the three common univariate EDAs for the binary decision variable to multi-valued decision variables.
We call these variants \emph{multi-valued EDAs}.
To this end, we brief{}ly discuss the binary case in \Cref{sec:binaryEDAs} before presenting our framework in \Cref{sec:framework}. In our presentation, we concentrate on the UMDA~\cite{MuhlenbeinP96} and then brief{}ly present the generalizations of the other two common univariate EDAs.

We note that for classic evolutionary algorithms, multi-valued decision variables have been discussed to some extent~\cite{DoerrJS11,DoerrP12,DoerrSW13foga, KotzingLW15,YuQZ15, LissovoiW16,DoerrDK18}. Due to the very different working principles, we could not see how these results help in designing and analyzing multi-valued EDAs.

\subsection{Binary EDAs}
\label{sec:binaryEDAs}

Binary EDAs refer to EDAs for \emph{pseudo-Boolean} optimization, that is, the optimization of functions $f\colon \{0, 1\}^n \to \R$.
This setting is a special case of optimizing $r$-valued fitness functions, for $r = 2$.
The probabilistic model of univariate EDAs in this domain is a length-$n$ vector~$p$ of probabilities (the \emph{frequency vector}), where the probability (the \emph{frequency}) at position $i \in [n]$ denotes the probability that a sample has a~$1$ at position~$i$, independent of the other positions.
Formally, for all $x, y \in \{0, 1\}^n$, it holds that $\Pr[x = y] = \prod_{i \in [n]} ({p_i}^{y_i} \cdot (1 - p_i)^{1 - y_i})$, where we assume that $0^0 = 1$.

Binary EDAs commonly take at least a parameter $\lambda \in \N_{\geq 1}$ (the \emph{population size}) as well as a pseudo-Boolean fitness function~$f$ as input and optimize~$f$ as follows:
Initially, the frequency vector~$p$ models the uniform distribution, that is, each frequency is~$1/n$.
Then, in an iterative manner, the algorithm produces~$\lambda$ samples (the \emph{population}) independently via~$p$, and it updates~$p$ based on these samples and their fitness.
This process is repeated until a user-defined termination criterion is met.

In order to prevent frequencies from only producing a single value (which is the case if a frequency is~$0$ or~$1$), after the frequency vector is updated, it is typically restricted to the interval $[1/n, 1 - 1/n]$.
That is, if the frequency is less than~$1/n$, it is set to~$1/n$, and if it is greater than $1 - 1/n$, it is set to $1 - 1/n$.
The extreme values of this interval are referred to as the \emph{borders}, and the value~$1/n$ is called the \emph{margin} of the algorithm.

\textbf{UMDA.}
\Cref{alg:UMDA} shows the \emph{univariate marginal distribution algorithm (UMDA)}~\cite{MuhlenbeinP96}, which is a well established binary EDA, both in the empirical~\cite{PelikanHL15} and the theoretical~\cite{DoerrN20} domain.
Next to the population size $\lambda \in \N_{\geq 1}$ and a fitness function, the UMDA also utilizes a parameter $\mu \in [\lambda]$, called the \emph{selection size}.
In each iteration, the UMDA selects~$\mu$ out of the~$\lambda$ samples that have the best fitness (breaking ties uniformly at random).
Each frequency is then set to the relative frequency of~$1$s at the respective position (\cref{line:umdaUpdate}).
Afterwards, the frequencies are restricted to lie within the frequency borders.

\begin{algorithm}[t]
    \caption{\label{alg:UMDA} The UMDA~\cite{MuhlenbeinP96} with parameters $\lambda \in \N_{\geq 1}$ and $\mu \in [\lambda]$, maximizing a pseudo-Boolean fitness function~$f$}
    $t \gets 0$\;
    $p^{(0)} \gets (\tfrac{1}{2})_{i \in \left[n\right]}$\; 
    \Repeat( \texttt{// iteration}~$t$)
    {\emph{termination criterion met}}
    {
        $P^{(t)} \gets$ population of~$\lambda$ individuals, independently sampled from~$p^{(t)}$\;
        $\{x^{(t, k)}\}_{k \in [\mu]} \gets$ multiset of~$\mu$ individuals from~$P^{(t)}$ with the highest fitness (breaking ties uniformly at random)\;
        \lFor{$i \in [n]$}
        {%
            $\overline{p}^{(t + 1)}_i \gets \frac{1}{\mu} \sum_{k \in [\mu]}x^{(t, k)}_i$%
            \label{line:umdaUpdate}%
        }
        $p^{(t + 1)} \gets$ values of $\overline{p}^{(t + 1)}$, restricted to $\left[\tfrac{1}{n}, 1 - \tfrac{1}{n}\right]$\;
        $t \gets t + 1$\;
    }
\end{algorithm}

\subsection{The Multi-Valued EDA Framework}
\label{sec:framework}
We propose a framework for EDAs for optimizing $r$-valued fitness functions.
We call the resulting EDAs $r$-valued EDAs.
Our framework closely follows the one presented in \Cref{sec:binaryEDAs}.
That is, an $r$-valued EDA starts with a probabilistic model initialized to represent the uniform distribution, and it then generates iteratively $\lambda \in \N_{\geq 1}$ samples independently, based on its model.
This model is then updated and afterwards restricted such that it does not contain the extreme probabilities~$0$ and~$1$.

The difference to the framework for binary EDAs lies in how the probabilistic model of $r$-valued EDAs is represented and how it is restricted from containing extreme probabilities.

\textbf{The probabilistic model.}
The probabilistic model of an $r$-valued EDA is an $n \times r$ matrix~$(p_{i, j})_{(i, j) \in [n] \times [0 .. r - 1]}$ (the \emph{frequency matrix}), where each row $i \in [n]$ forms a vector $p_i \coloneqq (p_{i, j})_{j \in [0 .. r - 1]}$ (the \emph{frequency vector at position~$i$}) of probabilities (the \emph{frequencies}) that sum to~$1$.
As in the binary case, samples from~$p$ are created independently for each position.
When creating an individual $x \in [0 .. r - 1]^n$, then, for all $i \in [n]$ and all $j \in [r - 1]$, the probability that~$x_i$ has value~$j$ is~$p_{i, j}$.
Formally, for all $x, y \in [0 .. r - 1]^n$, it holds that $\Pr[x = y] = \prod_{i \in [n]} \prod_{j \in [0 .. r - 1]} (p_{i, j})^{\bbone_{y_i = j}}$, where we assume that $0^0 = 1$.

The frequency matrix~$p$ is initialized such that each frequency is~$1/r$, representing the uniform distribution.
When performing an update to~$p$, it is important to make sure that each row sums to~$1$.

\textbf{Restricting the probabilistic model.}
\label{restr}
The aim of restricting the frequency matrix~$p$ is to clamp all frequencies, for some values $a, b \in [0, 1]$ (the \emph{lower} and \emph{upper border}, respectively) with $a \leq 1/r \leq b$, to $[a, b]$.
That is, if a frequency~$q$ is less than~$a$, it should be~$a$ after the restriction, and if it is greater than~$b$, it should be~$b$ afterwards.
For such a restriction, it is important for each row $i \in [n]$ that the frequency vector~$p_i$ sums to~$1$ after the restriction.
This process is not straightforward.
If $q \notin [a, b]$, and~$q$ is updated to $q' \in [a, b]$, then this creates a change in probability mass of $q' - q$.
Hence, simply updating~$q$ to~$q'$ can result in all frequencies of~$p_i$ summing to a value other than~$1$ after the restriction.

We address the problem above as follows.
To this end, let $a, b \in [0, 1]$ be the lower and upper border, respectively, with $a \leq 1/(r - 1) - 1/(r(r - 1))$ and $b = 1 - a(r - 1)$.
Further, let $i \in [n]$ be a row of the frequency matrix we wish to restrict, let $\overline{p}_i \in [0, 1]^n$ be the frequency vector after the update but before the restriction (with $\|\overline{p}_i\|_1 = 1$), and let $p^{+}_i \in [a, b]^n$ be the vector~$\overline{p}_i$ after clamping it to $[a, b]$ but before taking care that the frequencies sum to~$1$.
We define the \emph{restriction of~$\,\overline{p}_i$ to~$[a, b]$}, denoted by~$p'_i$, to be the vector where each frequency's share above~$a$ is reduced by the surplus of the probability relatively to the share above~$a$.
Formally, for all $j \in [0 .. r - 1]$, it holds that
\begin{align}
    \label{eq:restrictFrequencies}
    p'_{i, j} = (p^{+}_{i, j} - a)\frac{1 - ar}{\|p^+_i - (a)_{k \in [n]}\|_1} + a.
\end{align}
Note that $1 - ar = \|\overline{p}_i - (a)_{k \in [n]}\|_1$ denotes how much probability mass \emph{should} be in the frequency vector, above~$a$.
The resulting frequency vector~$p'_i$ sums to~$1$, since
\begin{align*}
    \sum\nolimits_{j \in [0 .. r - 1]} p'_{i, j} &= \frac{1 - ar}{\|p^+_i - (a)_{k \in [n]}\|_1} \sum\nolimits_{j \in [0 .. r - 1]} (p^{+}_{i, j} - a) + \sum\nolimits_{j \in [0 .. r - 1]} a\\
    &= 1 - ar + ar = 1.
\end{align*}
Further, each frequency is at least~$a$, since this value is added at the end of \cref{eq:restrictFrequencies} and since $p^{+}_{i, j} \geq a$ by definition of~$p^{+}_i$.
Last, since each frequency is at least~$a$ after restricting, the largest a frequency can be is $1 - (r - 1)a = b$.

In order to disallow the extreme frequencies~$0$ and~$1$ but to stay close to the binary case, we propose to choose the upper border as $1 - 1/n$.
Following our ideas above, this implies that the lower border is $1/((r - 1)n)$.
This is consistent with the binary case but generalizes to the $r$-valued domain.

We say that an EDA is \emph{without margins} if and only if the lower border is~$0$ and the upper border is~$1$.
That is, the restriction of the frequencies does not take place.

\textbf{\rUMDA.}
We generalize the UMDA (\Cref{alg:UMDA}) to the \rUMDA (\Cref{alg:rUMDA}), utilizing our framework.
This leads to the same generalization mentioned by Santana et~al.~\cite{SantanaORS02IntegerEDAs}.
Like the UMDA, the \rUMDA has three parameters, namely the population size $\lambda \in \N_{\geq 1}$, the selection size $\mu \in [\lambda]$, and the $r$-valued fitness function~$f$.
It also updates its frequencies analogously to the UMDA by choosing~$\mu$ best individuals from the population of size~$\lambda$ and then setting each frequency at position $i \in [n]$ for value $j \in [0 .. r - 1]$ to the relative frequency of value~$j$ at position~$i$ among the~$\mu$ best individuals (\cref{line:rUMDAUpdate}).
We note that this results in a valid frequency vector for each row $i \in [n]$, since
\begin{align*}
    \sum_{j \in [0 .. r - 1]} \frac{1}{\mu} \sum_{k \in [\mu]} \bbone_{x_{i}^{(t, k)}=j}
    = \frac{1}{\mu} \sum_{k \in [\mu]} \sum_{j \in [0 .. r - 1]} \bbone_{x_{i}^{(t, k)}=j}
    = \frac{1}{\mu} \sum_{k \in [\mu]} 1 = 1.
\end{align*}

\begin{algorithm}[t]
    \caption{\label{alg:rUMDA} The \rUMDA with parameters $\lambda \in \N_{\geq 1}$ and $\mu \in [\lambda]$, maximizing an $r$-valued fitness function~$f$}
    $t \gets 0$\;
    $ p^{(0)} \gets (\frac{1}{r})_{(i, j) \in [n] \times [0 .. r - 1]}$\; 
    \Repeat( \texttt{// iteration}~$t$)
    {\emph{termination criterion met}}
    {
        $P^{(t)} \gets$ population of~$\lambda $ individuals, independently sampled from~$p^{(t)}$\;
        $\{x^{(t, k)}\}_{k \in [\mu]} \gets$ multiset of~$\mu$ individuals from~$P^{(t)}$ with the highest fitness (breaking ties uniformly at random)\;
        \For{$(i, j) \in \left[n\right]\times \left[0..r-1\right]$}
        {%
            $\overline{p}_{i,j}^{(t+1)}\gets \frac{1}{\mu} \sum_{k\in \left[\mu\right]}\bbone_{x_{i}^{(t, k)}=j}$\;\label{line:rUMDAUpdate} 
        }
        $p^{(t + 1)} \gets$ restriction of~$\overline{p}^{(t+1)}$ to $\left[\tfrac{1}{(r-1)n}, 1 - \tfrac{1}{n}\right]$, as described in \cref{eq:restrictFrequencies}\;
        $t \gets t + 1$\;
    }
\end{algorithm}

\textbf{$r$-PBIL.}
Another popular univariate EDA is \emph{population-based incremental learning} (PBIL~\cite{Baluja94}).
It operates very similarly to the UMDA, with the only difference being in how it performs an update.
In contrast to the UMDA, the PBIL does not set a frequency to the relative frequency of respective values at a position but, instead, computes the convex combination of the relative frequency with the current frequency value in its frequency vector.
To this end, it utilizes a parameter $\rho \in [0, 1]$, the \emph{scaling factor}.

We generalize the PBIL to the $r$-PBIL (\Cref{alg:rPBIL}).
Each frequency vector of the $r$-PBIL sums to~$1$ (before the restriction) because it is a convex combination of the $r$-UMDA's update (which sums to~$1$) and the current frequency vector (which also sums to~$1$).

\begin{algorithm}[t]
    \caption{\label{alg:rPBIL} The $r$-PBIL with parameters $\lambda \in \N_{\geq 1}$, $\mu \in [\lambda]$, and $\rho \in [0, 1]$, maximizing an $r$-valued fitness function~$f$}
    $t \gets 0$\;
    $ p^{(0)} \gets (\frac{1}{r})_{(i, j) \in [n] \times [0 .. r - 1]}$\; 
    \Repeat( \texttt{// iteration}~$t$)
    {\emph{termination criterion met}}
    {
        $P^{(t)} \gets$ population of~$\lambda $ individuals, independently sampled from~$p^{(t)}$\;
        $\{x^{(t, k)}\}_{k \in [\mu]} \gets$ multiset of~$\mu$ individuals from~$P^{(t)}$ with the highest fitness (breaking ties uniformly at random)\;
        \For{$(i, j) \in \left[n\right]\times \left[0..r-1\right]$}
        {%
            $\overline{p}_{i,j}^{(t+1)}\gets  (1 - \rho) p_{i, j}^{(t)} + \frac{\rho}{\mu} \sum_{k\in \left[\mu\right]}\bbone_{x_{i}^{(t, k)}=j}$\; 
        }
        $p^{(t + 1)} \gets$ restriction of~$\overline{p}^{(t+1)}$ to $\left[\tfrac{1}{(r-1)n}, 1 - \tfrac{1}{n}\right]$, as described in \cref{eq:restrictFrequencies}\;
        $t \gets t + 1$\;
    }
\end{algorithm}

\textbf{$r$-cGA.}
Another popular univariate EDA is the \emph{compact genetic algorithm} (cGA~\cite{HarikLG99}).
The cGA only has a single parameter $K \in \R_{> 0}$, the \emph{hypothetical population size}, and it creates only two samples each iteration.
It ranks these two samples by fitness and then adjusts each frequency by~$\frac{1}{K}$ such that the frequency of the value of the better sample is increased and that of the worse sample decreased.

We generalize the cGA to the $r$-cGA (\Cref{alg:rcGA}).
Each frequency vector of the $r$-cGA sums to~$1$ after the update (before the restriction) because exactly one entry is increased by~$\frac{1}{K}$ and exactly one value is decreased by this amount (noting that this can be the same frequency, in which case no change is made overall).

\begin{algorithm}[t]
    \caption{\label{alg:rcGA} The $r$-cGA with parameter $K \in \R_{> 0}$, maximizing an $r$-valued fitness function~$f$}
    $t \gets 0$\;
    $ p^{(0)} \gets (\frac{1}{r})_{(i, j) \in [n] \times [0 .. r - 1]}$\; 
    \Repeat( \texttt{// iteration}~$t$)
    {\emph{termination criterion met}}
    {
        $x^{(t, 1)}, x^{(t, 2)} \gets$ two individuals, independently sampled from~$p^{(t)}$\;
        $y^{(t, 1)} \gets$ individual with the higher fitness from $\{x^{(t, 1)}, x^{(t, 2)}\}$ (breaking ties uniformly at random)\;
        $y^{(t, 2)} \gets$ individual from $\{x^{(t, 1)}, x^{(t, 2)}\} \setminus \{y^{(t, 1)}\}$\;
        \For{$(i, j) \in \left[n\right]\times \left[0..r-1\right]$}
        {%
            $\overline{p}_{i,j}^{(t+1)}\gets p_{i,j}^{(t)} +  \bigl(\bbone_{y^{(t, 1)}_{i, j} = j} - \bbone_{y^{(t, 2)}_{i, j} = j}\bigr) \frac{1}{K}$\; 
        }
        $p^{(t + 1)} \gets$ restriction of~$\overline{p}^{(t+1)}$ to $\left[\tfrac{1}{(r-1)n}, 1 - \tfrac{1}{n}\right]$, as described in \cref{eq:restrictFrequencies}\;
        $t \gets t + 1$\;
    }
\end{algorithm}

\section{Genetic Drift}
\label{sec:geneticDrift}

We prove an upper bound on the effect of genetic drift for $r$-valued EDAs (\Cref{th:neutral_bound}) in a similar fashion as Doerr and Zheng~\cite{DoerrZ20tec} for binary decision variables.
This allows us to determine parameter values for EDAs that avoid the usually unwanted effect of genetic drift.
The main novelty of our result over that by Doerr and Zheng~\cite{DoerrZ20tec} is that we use a slightly technical martingale concentration result due to McDiarmid~\cite{McDiarmid98} that allows one to profit from small variances.
Such an approach is necessary. If one directly applies the methods presented by Doerr and Zheng~\cite{DoerrZ20tec}, one obtains estimates for the genetic drift times that are by a factor of $\Theta(r)$ lower than ours (that is, the genetic drift effect appears $r$ times stronger).

In \Cref{sec:introGeneticDrift,sec:neutralSection}, we first present a general introduction to the phenomenon of genetic drift.
In \Cref{sec:upperBoundOnGeneticDrift}, we then prove a concentration result on neutral positions (\Cref{th:neutral_bound}).
Last, in \Cref{sec:weakPreference}, we consider the setting of weak preference.

\subsection{Introduction to Genetic Drift}
\label{sec:introGeneticDrift}

In EDAs, \emph{genetic drift} means that a frequency does not reach  or approach one of the extreme values~$0$ or~$1$ because of a clear signal from the objective function but  due to random fluctuations from the stochasticity of the process.

While there is no proof that genetic drift is always problematic, the general opinion is that this effect should better be avoided. This is supported by the following observations and results: (i)~When genetic drift is strong, many frequencies (in the binary case) approach the extreme values~$0$ and~$1$ and, consequently, the behavior of the EDA comes close to the one of a mutation-based EA, so the advantages of an EDA might be lost. (ii)~The vast majority of the runtime results for EDAs, especially those for harder scenarios like noise~\cite{FriedrichKKS17} or multimodality~\cite{Doerr21cgajump}, have only been shown in regimes with low genetic drift. (iii)~For some particular situations, a drastic performance from genetic drift was proven. For example, the UMDA with standard selection pressure but small population size $\lambda \in \Omega(\ln(n)) \cap o(n)$ has a runtime exponential in~$\lambda$ on the \DLB problem~\cite{LehreN19foga}. In contrast, when the population size is large enough to prevent genetic drift, here $\lambda = \Omega(n \ln(n))$, then the runtime drops to $O(\lambda n)$ with high probability.

Genetic drift in EDAs has been studied explicitly since the ground-breaking works of Shapiro~\cite{Shapiro02,Shapiro05,Shapiro06}, and it appears implicitly in many runtime analyses such as~\cite{Droste05,Witt18,Witt19,SudholtW19,LenglerSW21,DoerrK21ecj}. Experimental evidences for the negative impact of genetic drift can further be found in~\cite{KrejcaW20bookchapter,DoerrZ20tec,NeumannSW22}.
The most final answer to the genetic-drift problem for univariate EDAs, including clear suggestions to choose the parameters as to avoid genetic drift, was given by Doerr and Zheng~\cite{DoerrZ20tec}. In the case of the UMDA (and binary decision variables, that is, the classic model), their work shows that a neutral frequency (defined in \Cref{sec:neutralSection}) stays with high probability in the middle range $[0.25,0.75]$ for the first $T$ iterations if $\mu = \omega(T)$. This bound is tight. When regarding $n$ frequencies together, a value of $\mu = \Omega(T \ln(n))$ with implicit constant computable from~\cite[Theorem~$2$]{DoerrZ20tec} ensures with high probability that all frequencies stay in the middle range for at least $T$ iterations. Hence these bounds give a clear indication how to choose the selection size $\mu$ when aiming to run the UMDA for a given number of iterations. We note that the quantification of genetic drift can also be used to design automated ways to choose parameters, see the work by Zheng and Doerr~\cite{ZhengD23}, when no a-priori estimate on~$T$ is available.

Given the importance of a good understanding of genetic drift, we now analyze genetic drift for multi-valued EDAs, more specifically, for the \rUMDA. We are optimistic that, analogous to the work by Doerr and Zheng~\cite{DoerrZ20tec}, very similar arguments can be applied for other main univariate EDAs.

\subsection{Martingale Property of Neutral Positions}
\label{sec:neutralSection}

Genetic drift is usually studied via \emph{neutral} positions of a fitness function. Let~$f$ be an $r$-valued fitness function. We call a position $i \in \left[n\right]$ (as well as, for an individual $x \in [0 .. r - 1]^n$, its corresponding variable $x_i$ and the associated frequencies of an EDA) \emph{neutral} (w.r.t. to~$f$) if and only if, for all $x \in [0 .. r - 1]^n$, the value~$x_i$ has no influence on the value of~$f$, that is, if and only if for all individuals $x,x' \in \left[0..r-1\right]^{n}$ such that for all $j\in \left[n\right] \setminus \{i\}$ it holds that $x_{j}=x'_{j}$, we have $f(x)=f(x')$.

An important property of neutral variables that we capitalize on in our analysis of genetic drift is that their frequencies in typical EDAs without margins form martingales~\cite{DoerrZ20tec}. This observation extends the corresponding one for EDAs for binary representations. We make this statement precise for the \rUMDA.

\begin{lemma}
    \label{lem:neutralFrequencyIsMartingale}
    Let~$f$ be an $r$-valued position, and let $i \in [n]$ be a neutral position of~$f$.
    Consider the \rUMDA without margins optimizing~$f$.
    For each $j \in \left[0..r-1\right]$, the frequencies $(p_{i,j}^{(t)})_{t\in \N}$ are a martingale.
\end{lemma}
\begin{proof}
    Let $j \in [0 .. r - 1]$.
    Since the algorithm has no margins, in each iteration $t \in \N$, no restriction takes place, so it holds that $p_{i,j}^{(t+1)}= \frac{1}{\mu} \sum_{k\in \left[\mu\right]} \bbone_{x_{i}^{(t, k)}=j}$. Since~$i$ is neutral, the selection of the~$\mu$ best individuals is not affected by the values at position~$i$ of the $\lambda$ samples.
    Consequently, for each $k \in \left[\mu\right]$, the value~$x_i^{(t, k)}$ follows a Bernoulli distribution with success probability $p^{(t)}_{i,j}$.
    Hence, $\E[\bbone_{x_{i}^{(t, k)}=j}\mid  p_{i,j}^{(t)}] = p^{(t)}_{i,j}$.
    Further, by linearity of expectation, we get
    \begin{align*}
        \E\left[p_{i,j}^{(t+1)}\mid p_{i,j}^{(t)}\right]
        =\frac{1}{\mu}\sum\nolimits_{k\in \left[\mu\right]}  \E\left[\bbone_{x_{i}^{(t, k)}=j} \ \big|\ p_{i,j}^{(t)}\right]
        =\frac{1}{\mu}\sum\nolimits_{k\in \left[\mu\right]}p_{i,j}^{(t)}=p_{i,j}^{(t)},
    \end{align*}
    proving the claim.
\end{proof}

As in previous works on genetic drift, the martingale property of neutral frequencies allows to use strong martingale concentration results. Since in our setting the frequencies start at a value of $\frac 1r$, we can only tolerate smaller deviations from this value, namely up to $\frac 1{2r}$ in either direction. With the methods of Doerr and Zheng~\cite{DoerrZ20tec}, this reduces the genetic drift by a factor of $\Theta(r^2)$. We therefore use a stronger martingale concentration result, namely \cite[Theorem~$3.15$]{McDiarmid98}, which allows to exploit the lower sampling variance present at frequencies in $\Theta(\frac 1r)$.
We note that we adjust the theorem by incorporating comments by McDiarmid, especially \cite[eq.~($41$)]{McDiarmid98}, mentioning that the absolute value in eq.~$(41)$ should be around the sum, not around the maximum, as also observed by Doerr and Zheng~\cite{DoerrZ20tec}.

\begin{theorem}[Martingale concentration result based on the variance {\cite[Theorem~3.15 and eq.~($41$)]{McDiarmid98}}]
    \label{th:mcdiarmid}
    Let $(X_{t})_{t \in \N}$ be a martingale with respect to a filtration $(\calF_t)_{t \in \N}$.
    Further, for all $t \in \N_{\geq 1}$,  denote the deviation by $\mathrm{dev}_{t} \coloneqq |X_t - X_{t - 1}|$.
    In addition, let $b=\sup_{t \in \N} \mathrm{dev}_{t}$, and assume that~$b$ is finite.
    Last, for all $t \in \N$, let $\hat{v}_t = \sup \sum_{s \in [t]} \Var[X_s - X_{s -1} \mid \calF_{s - 1}]$.
    Then for all $t \in \N$ and all $\varepsilon \in \R_{\geq 0}$, it holds that
    \begin{align*}
        \Pr\left[\max\nolimits_{s \in [0 .. t]} |X_s - \E\left[X_0\right]| \geq \varepsilon \right] \leq 2 \exp\left(-\frac{\varepsilon^{2}}{2\hat{v}_t + 2b\varepsilon/3}\right).
    \end{align*}
\end{theorem}

\subsection{Upper Bound on the Genetic-Drift Effect of a Neutral Position}
\label{sec:upperBoundOnGeneticDrift}

By utilizing \Cref{th:mcdiarmid}, we show for how long the frequencies of the \rUMDA at neutral positions stay concentrated around their initial value of~$\frac{1}{r}$.
\begin{theorem}
    \label{th:neutral_bound}
    Let~$f$ be an $r$-valued fitness function, and let $i \in [n]$ be a neutral position of~$f$.
    Consider the \rUMDA optimizing~$f$.
    Let $T \in \N$ and $j\in [0..r-1]$.
    Then
    \begin{align*}
        \Pr\left[\max\nolimits_{s \in [0 .. T]}\ \left|p_{i,j}^{(s)}-\frac{1}{r}\right| \geq \frac{1}{2r}\right]\leq 2\exp\left(-\frac{\mu}{12Tr+(4/3)r}\right).
    \end{align*}
\end{theorem}

\begin{proof}
    We apply the same proof strategy as in the proof of \cite[Theorem~$1$]{DoerrZ20tec}.
    That is, we aim to apply \Cref{th:mcdiarmid}.
    Naturally, one would apply the theorem to the sequence of frequencies~$(p^{(t)}_{i, j})_{t \in \N}$.
    However, since the deviation of~$p_{i, j}$ is very large, namely~$1$, we consider instead a more fine-grained process $(Z_t)_{t \in \N}$, which, roughly speaking, splits each iteration of the \rUMDA into~$\mu$ sections, each of which denotes that an additional sample is added to the update.
    Formally, for all $t \in \N$ and $a \in [0 .. \mu - 1]$, let
    \begin{align*}
        Z_{t\mu+a}=p_{i,j}^{(t)}\left(\mu-a\right)+\sum\nolimits_{k\in \left[a\right]}\bbone_{x_{i}^{(t+1, k)}=j}.
    \end{align*}
    Note that, for all $t \in \N_{\geq 1}$, it holds that $Z_{t\mu} = \mu p^{(t)}_{i, j}$.
    Thus, the natural filtration $(\calF_t)_{t \in \N}$ of~$Z$ allows us to measure~$p_{i, j}$.

    In order to apply \Cref{th:mcdiarmid}, we check that its assumptions are met.
    To this end, we first show that~$Z$ is a martingale.
    Since~$i$ is neutral, the selection of the~$\mu$ best individuals is not affected by the values at position~$i$ of the~$\lambda$ samples. Consequently, for all $k \in \left[\mu\right]$, the random variable $x_{i}^{(t, k)}$ follows a Bernoulli distribution with success probability~$p^{(t)}_{i,j}$.
    Thus, we get for all $t \in \N$ and $a \in [0 .. \mu - 2]$ that
    \begin{align}
        \label{eq:neutral_bound:differenceOne}
        \E\left[Z_{t\mu + a + 1} - Z_{t\mu + a}\mid \calF_{t\mu + a}\right]
        = -p^{(t)}_{i, j} + \E[\bbone_{x_{i}^{(t, a + 1)} = j} \mid \calF_{t\mu + a}]
        = 0,
    \end{align}
    and further, by the definition of $p^{(t + 1)}_{i, j}$, that
    \begin{align}
        \notag
        &\E\left[Z_{(t + 1)\mu} - Z_{t\mu + \mu - 1}\mid \calF_{t\mu + \mu - 1}\right]\\
        \notag
        &\quad= \mu\E[p^{(t + 1)}_{i, j} \mid \calF_{t\mu + \mu - 1}] - p^{(t)}_{i, j} - \E\bigl[\sum\nolimits_{k\in \left[\mu - 1\right]}\bbone_{x_{i}^{(t, k)}=j} \ \big|\ \calF_{t\mu + \mu - 1}\bigr]\\
        \notag
        &\quad= \sum\nolimits_{k\in \left[\mu\right]} \E[\bbone_{x_{i}^{(t, k)} = j} \mid \calF_{t\mu + \mu - 1}] - p^{(t)}_{i, j} - \sum\nolimits_{k\in \left[\mu - 1\right]} \E[\bbone_{x_{i}^{(t, k)} = j} \mid \calF_{t\mu + \mu - 1}]\\
        \label{eq:neutral_bound:differenceTwo}
        &\quad= \E[\bbone_{x_{i}^{(t, \mu)} = j} \mid \calF_{t\mu + \mu - 1}] - p^{(t)}_{i, j} = 0,
    \end{align}
    showing that~$Z$ is a martingale.

    We take an alternative view of the event $\{\max\nolimits_{s \in [0 .. T]}\ |p_{i,j}^{(s)}-\frac{1}{r}| \geq \frac{1}{2r}\}$, whose probability we aim to bound.
    Note that this event is equivalent to $\{\exists s \in [0 .. T]\colon |p_{i,j}^{(s)}-\frac{1}{r}| \geq \frac{1}{2r}\}$.
    A superset of this event is the event where we stop at the first iteration such that the inequality holds.
    To this end, let $S = \inf\{t \in \N \mid Z_t \notin [\frac{\mu}{2r}, \frac{3\mu}{2r}]\}$ be a stopping time (with respect to~$\calF$).
    From now on, we consider the stopped process~$\widetilde{Z}$ of~$Z$ with respect to~$S$.
    That is, for all $t \in \N$, it holds that $\widetilde{Z}_t = Z_{\min\{t, S\}}$.
    Since~$Z$ is a martingale, so is~$\widetilde{Z}$.

    Let $t \in \N$, and let~$Y_t$ be a Bernoulli random variable with success probability~$p^{(\lfloor t/\mu\rfloor)}_{i, j}$ that is $\calF_t$-measurable.
    Note that by \cref{eq:neutral_bound:differenceOne,eq:neutral_bound:differenceTwo}, disregarding the expected values, by \cref{eq:neutral_bound:stoppedDifference}, it holds that
    \begin{align}
        \label{eq:neutral_bound:stoppedDifference}
        \widetilde{Z}_{t + 1} - \widetilde{Z}_t = (Y_t - p^{(\lfloor t/\mu\rfloor)}_{i, j}) \cdot \bbone_{t < S}.
    \end{align}
    Thus, the maximum deviation~$b$ of~$\widetilde{Z}$ is~$1$.
    Further, let~$\hat{v}_t$ denote the sum of variances, as defined in \Cref{th:mcdiarmid}.
    Then, since~$p^{(\lfloor t/\mu\rfloor)}_{i, j}$ and~$\bbone_{t < S}$ are $\calF_t$-measurable and since, due to~$\widetilde{Z}$ being stopped, it holds that $p^{(\lfloor t/\mu\rfloor)}_{i, j} \cdot \bbone_{t < S} \in [\frac{1}{2r}, \frac{3}{2r}]$, we get
    \begin{align*}
        \Var\left[\widetilde{Z}_{t + 1} - \widetilde{Z}_t \mid \calF_{t} \right]
        = \Var\left[Y_t \cdot \bbone_{t < S} \mid \calF_{t}\right]
        = p^{(\lfloor t/\mu\rfloor)}_{i, j} \left(1-p^{(\lfloor t/\mu\rfloor)}_{i, j}\right) \cdot \bbone_{t < S} \leq  \frac{3}{2r}.
    \end{align*}
    Hence, $\hat{v}_t \leq \frac{3t}{2r}$.

    Let~$\widetilde{p}$ denote the stopped process of~$p_{i, j}$ with respect to~$S$.
    Applying \Cref{th:mcdiarmid} with $t = \mu T$ and our estimates above, noting that $\widetilde{Z}_0 = \frac{\mu}{r}$, yields
    \begin{align*}
        &\Pr\left[\max_{s \in [0 .. T]} \left|\widetilde{p}_s - \frac{1}{r}\right| \geq \frac{1}{2r} \right]
        = \Pr\left[\max_{s \in [0 .. T]} |\widetilde{p}_s - \E[\widetilde{p}_0]| \geq \frac{1}{2r} \right]\\
        &=\Pr\left[\max_{s \in [0 .. T]} \frac{1}{\mu}|\widetilde{Z}_{s\mu} - \E[\widetilde{Z}_0]| \geq \frac{1}{2r} \right]
        \leq\Pr\left[\max_{s \in [0 .. t]} |\widetilde{Z}_s - \E[\widetilde{Z}_0]| \geq \frac{\mu}{2r} \right]\\
        &\leq 2 \exp\left(-\frac{(\mu/(2r))^{2}}{2 \cdot 3\mu T/(2r) + (2/3)\mu/(2r)}\right)
        = 2 \exp\left(-\frac{\mu^2}{12Tr + (4/3) r}\right).
    \end{align*}
    Since we only need to consider the stopped process, as explained above, and since~$\widetilde{p}$ is identical to~$p_{i, j}$ until the process stops, the result follows.
\end{proof}

\subsection{Upper Bound for Positions with Weak Preference}
\label{sec:weakPreference}
A position is rarely neutral for a given fitness function. However, we prove that the results on neutral positions translate to positions where one value is better than all other values. This is referred to as \emph{ weak preference}.
Formally, we say that an $r$-valued fitness function~$f$ has a \emph{weak preference for a value $j \in \left[0..r-1\right]$ at a position $i \in \left[n\right]$} if and only if, for all $x_{1},...,x_{n} \in \left[0..r-1\right]$, it holds that
\begin{align*}
    f\left(x_{1},..,x_{i-1},x_{i},x_{i+1},...,x_{n}\right) \leq f\left(x_{1},..,x_{i-1},j,x_{i+1},...,x_{n}\right).
\end{align*}

We now adapt Lemma~$7$ by Doerr and Zheng~\cite{DoerrZ20tec} to the \rUMDA.

\begin{theorem}
    \label{th:domination}
    Consider two r-valued fitness functions $f, g$ to optimize using the \rUMDA, such that without loss of generality, the first position of f weakly prefers 0 and the first position of g is neutral.

    Let $p$ correspond to the frequency matrix of $f$ and $q$ to the frequency matrix of $g$, both defined by the \rUMDA. Then, for all $t \in \N$, it holds that $q_{1,0}^{(t)}\preceq p_{1,0}^{(t)}$.
\end{theorem}

\begin{proof}
    We prove our claim by induction on the number of iterations $t$.
    For the base case $t=0$, all frequencies are~$1/r$.
    Hence, $q_{1,0}^{(0)}\preceq p_{1,0}^{(0)}$.

    For the induction step, let $t\in \N_{\geq 1}$ and let $j \in \left[0..r-1\right]$.
    Further, let $Y_j \sim \mathrm{Bin} \left(\mu,q_{0,j}^{(t)}\right)$.
    Since~$0$ is a neutral position of~$g$, the selection of the~$\mu$ best individuals is not affected by the values at position~$0$ of the $\lambda$ samples.
    Thus, $q_{1,j}^{(t+1)}=\frac{1}{\mu}Y$.
    Further, since~$f$ weakly prefers~$0$s, defining $Y'_j \sim \mathrm{Bin} \left(\mu,p_{0,j}^{(t)}\right)$, it holds that $p_{1,j}^{t+1}\gtrsim\frac{1}{\mu}Y'$.

    Analogously to Doerr and Zheng~\cite{DoerrZ20tec}, we note that since $p_{1,0}^{(t)}$ stochastically dominates $q_{1,0}^{(t)}$ by induction hypothesis, there exists a coupling of the two probability spaces that describe the states of the two algorithms at iteration~$t$ in such a way that $p_{1,0}^{(t)} \geq q_{1,0}^{(t)}$ for any point~$w$ in the coupling probability space.
    For such a $w$, it then follows that $Y_j \preceq Y'_j$, as the success probability of the former is bounded from above by that of the latter.
    Hence, $q_{1,j}^{(t+1)}=\frac{1}{\mu}Y\preceq \frac{1}{\mu}Y'\preceq p_{1,j}^{(t+1)}$, which proves the claim.
\end{proof}

We now apply \Cref{th:domination} and extend \Cref{th:neutral_bound} to positions with weak preference.
\begin{theorem}
    \label{th:all_bound}
    Let~$f$ be an $r$-valued fitness function with a weak preference for~$0$ at position $i \in [n]$.
    Consider the \rUMDA optimizing~$f$.
    Let $T \in \N$.
    Then
    \begin{gather}
        \Pr\left[\min\nolimits_{s \in [0 .. T]} p_{i,0}^{(s)} \leq \frac{1}{2r}\right]\leq 2\exp\left(-\frac{\mu}{12Tr+ (4/3)r}\right).
    \end{gather}
\end{theorem}

\begin{proof}
    Let~$g$ be an $r$-valued fitness function with neutral position~$i$.
    Let~$q$ be the frequency matrix of the \rUMDA optimizing~$g$.
    By \Cref{th:domination}, it follows for all $s \in \N$ that $p_{i,0}^{(s)}$ stochastically dominates $q_{i,0}^{(s)}$.
    Applying \Cref{th:neutral_bound} to~$g$ for position~$i$, we have
    \begin{align*}
        \Pr\left[\min\nolimits_{s \in [0 ..T]} q_{i,0}^{(s)} \leq \frac{1}{2r}\right]\leq 2\exp\left(-\frac{\mu}{12Tr + (4/3)r}\right).
    \end{align*}
    Using the stochastic domination yields the tail bound for~$f$.
\end{proof}

\section{Runtime Analysis of the \texorpdfstring{\rUMDA}{r-UMDA}}
\label{sec:runtimeAnalysis}

We analyze the runtime of the \rUMDA (\Cref{alg:rUMDA}) on an $r$-valued variant of \leadingones. We start by describing the previous runtime results of EDAs on \leadingones (\Cref{sec:runtimePreviousWork}), then define the \rLO problem formally (\Cref{sec:leadingOnes}), and finally state and prove our main result (\Cref{thm:rUMDAonrLO}, \Cref{sec:runtimeResult}).

\subsection{Previous Runtime Analyses of EDAs on \texorpdfstring{\leadingones}{LeadingOnes}}
\label{sec:runtimePreviousWork}

In contrast to \onemax (another popular theory benchmark function), \LO is not that extensively studied for EDAs.
This is surprising, as \LO is interesting as a benchmark for univariate EDAs, since the function introduces dependencies among the different positions of a bit string, but the model of univariate EDAs assumes independence.
However, since \LO only has a single local maximum, known runtime results are rather fast.

In a first mathematical runtime analysis of an EDA, however, using the unproven no-error-assumption (which essentially states that there is no genetic drift), it was shown that the UMDA optimizes the \leadingones benchmark in expected time $O(\lambda n)$.
This was made rigorous by Chen et~al.~\cite{ChenTCY10} with a proof that the UMDA with population size $\Omega(n^{2+\eps})$ optimizes \leadingones in time $O(\lambda n)$ with high probability.
Here the relatively large required population stems from the, then, incomplete understanding of genetic drift.

In a remarkable work~\cite{DangL15}, Dang and Lehre prove a runtime of $O(n \lambda \ln(\lambda) + n^2)$, only assuming that the sample size~$\lambda$ is at least logarithmic.
Hence this result applies both to regimes without and with genetic drift.
In the regime with genetic drift, however, the dependence on $\lambda$ is slightly worse than in the result by Chen et~al.~\cite{ChenTCY10}.
This was improved by Doerr and Krejca~\cite{DoerrK21tcs}, where an $O(n \lambda \ln(\lambda))$ upper bound was shown for the whole regime $\lambda = \Omega(n \ln(n))$ of low genetic drift.
More precisely, when $\mu = \Omega(n \ln(n))$ and $\lambda = \Omega(\mu)$, both with sufficiently large implicit constants, then the runtime of the UMDA on \leadingones is $O(n \lambda \ln(\frac \lambda \mu))$ with high probability.
We note that the analysis by Doerr and Krejca~\cite{DoerrK21tcs} is technically much simpler than the previous ones, in particular, it avoids the complicated level-based method used by Dang and Lehre~\cite{DangL15}.
We note that also lower bounds~\cite{LehreN19gecco,DoerrK21tcs} and runtimes in the presence of noise have been regarded.
Since we have no such results, we refer to the original works.

Besides the UMDA, \LO was considered in the analysis of newly introduced univariate EDAs.
Interestingly, each of these algorithms optimizes \LO in $O(n \ln(n))$ with high probability.
This runtime is faster by a factor of $n / \ln(n)$ when compared to classical EAs, and it suggests that \LO is a rather easy problem for EDAs.
Friedrich, Kötzing, and Krejca~\cite{FriedrichKK16} proved the first of these results for their \emph{stable compact genetic algorithm} (scGA), which introduces an artificial bias into its update process that is overcome by the \LO function.
However, it was later proven that the scGA fails on the typically easy \onemax function~\cite{DoerrK20tec}, highlighting that the scGA is not a good EDA in general.

The next result was proven by Doerr and Krejca~\cite{DoerrK20tec}, who introduce the \emph{significance-based compact genetic algorithm} (sig-cGA).
The sig-cGA saves a history of good individuals and only updates a frequency when the number of bits in the history of that position significantly deviates from its expectation.
This algorithm also performs well on \onemax.

The last result was proven recently by Ajimakin and Devi~\cite{AjimakinD23}, who introduce the \emph{competing genes evolutionary algorithm} (cgEA).
The cgEA utilizes the Gauss--Southwell score as a quality metric for the positions of its samples.
Iteratively, it picks the position~$i$ with the best score and creates a new population by letting each individual of the previous population compete against a copy of it where the bit at position~$i$ is flipped.
Based on the best individuals created this way, the frequency at position~$i$ is immediately set to either~$0$ or~$1$, whichever value turns out to be better.
This approach works very well for a variety of theory benchmarks, as proven by the authors.


\subsection{The \texorpdfstring{\rLO}{r-LeadingOnes} Benchmark}
\label{sec:leadingOnes}

The \rLO function (\cref{eq:rLO}) is a generalization of the classical \LO benchmark~\cite{Rudolph97} from the binary to the multi-valued domain.
Before we define the generalization, we briefly present the \LO function.

\textbf{\LO.}
\LO~\cite{Rudolph97} is one of the most commonly mathematically analyzed benchmark functions, both in the general domain of evolutionary computation~\cite{DoerrN20} as well as in the domain of EDAs~\cite{KrejcaW20bookchapter}.
For a bit string of length $n \in \N_{\geq 1}$, it returns the number of consecutive~$1$s, starting from the leftmost position.
Formally, $\LO\colon \{0, 1\}^n \to [0 .. n]$ is defined as $x \mapsto \sum_{i \in [n]} \prod_{j \in [i]} x_i$.
The function has a single local maximum at the all-$1$s string, which is also its global maximum.

\textbf{\texorpdfstring{\rLO}{\rLO}.}
Inspired by \LO from the binary domain, we define $r\textrm{-\LO}\colon [0 .. r - 1]^n \to [0 .. n]$ as the function that returns the number of consecutive~$0$s, starting from the leftmost position.
Formally,
\begin{equation}
    \label{eq:rLO}
    r\textrm{-\LO}\colon x \mapsto \sum\nolimits_{i \in [n]} \prod\nolimits_{j \in [i]} \bbone_{\{x_j = 0\}}.
\end{equation}
In contrast to the binary case, the single local optimum of \rLO is the all-$0$s string, which is also its global optimum.

\subsection{Runtime Results}
\label{sec:runtimeResult}
We analyze the runtime of the \rUMDA (\Cref{alg:rUMDA}) on the \rLO benchmark (\cref{eq:rLO}) in the regime with low genetic drift.
For the upper bound (\Cref{thm:rUMDAonrLO}), compared to the binary case~\cite[Theorem~$5$]{DoerrK21tcs}, we get an extra factor of order~$r \ln(r)^2$ in the runtime.
The factor of~$r$ is a result of the increased waiting time to see a certain position out of~$r$.
The factor of~$\ln(r)^2$ stems from the choice to stay in the regime with low genetic drift as well as for the time it takes a frequency to get to the upper border.
For the lower bound, (\Cref{thm:rUMDALowerBound}), compared to the binary case~\cite[Theorem~$6$]{DoerrK21tcs}, we get an extra factor of order~$r \ln(r)$.

Our two bounds differ by a factor in the order of~$\ln(r)$ (for polynomial population sizes).
We believe that our lower bound is missing a factor of~$\ln(r)$, as we currently do not account for the time it takes a frequency to get from its starting value~$\frac{1}{r}$ to $1 - \frac{1}{n}$ for this bound.

We prove the upper bound in \Cref{sec:upperBound} and the lower bound in \Cref{sec:lowerBound}.
Both bounds are a generalization of the binary case.

\subsubsection{Upper Bound}
\label{sec:upperBound}
Our upper bound shows that the number of iterations until an optimum is found for the first time is almost linear in~$\lambda$ and in~$n$, only adding a factor in the order of $\ln(r)$.
\begin{theorem}
    \label{thm:rUMDAonrLO}
    Let $s \in \R_{\geq 1}$.
    Consider the \rUMDA optimizing \rLO with $ \lambda \geq 3 se\mu$, $\mu \geq 24 (n+1) r \ln(n) (1+\ln_{2s}(r))$, and $n\geq 4r$. Then with a probability of at least $1 - \frac{2}{n} - \ln_{2s}(2r) n^{2 -0.5n}$, the frequency vector corresponding to the value~$0$ converges to $(1-\frac{1}{n})_{i\in \left[n\right]}$ in $n \ln_{2s}(2r)$ iterations.

    This implies that after $\lambda n\ln_{2s}(2r)$ fitness function evaluations, the \rUMDA samples the optimum with the success probability above.
\end{theorem}

The basic premise for our proof is that for the entirety of the considered iterations, frequencies corresponding to the value $0$ remain above a given threshold since \rLO weakly prefers~$0$ at all positions. We define this threshold as $\frac{1}{2r}$, and we show that in a sequential manner, position by position, the frequencies corresponding to $0$ are brought to $1-\frac{1}{n}$ within a given number of iterations until all positions are covered.

First, we provide a guarantee on the concentration of all the probabilities during the entirety of the algorithm's runtime, in a way to avoid genetic drift and to remain above a minimal threshold for all frequencies.
\begin{lemma}
    \label{lem:frequencyDoesNotGetLow}
    Let $s \in \R_{\geq 1}$.
    Consider the \rUMDA with $\lambda \geq \mu \geq 24 (n+1) r \ln(n) (1+\ln_{2s}(r))$ optimizing a function that weakly prefers $0$ at every position. Then with a probability of at least $1-\frac{2}{n}$, for each $i\in \left[n\right]$, the frequency $p_{i,0}^{(t)}$ remains above $\frac{1}{2r}$ for the first $n (1+\ln_{2s}(r))$ iterations.
\end{lemma}
\begin{proof}
    By \Cref{th:all_bound} with $T=n (1+\ln_{2s}(r))$, we have for all $i \in \left[n\right]$ that
    \begin{align*}
        \Pr\left[\min_{k=1,...,T} p_{i,0}^{(k)} \leq \frac{1}{2r}\right] &\leq 2\exp\left(-\frac{\mu}{12n (1+\ln_{2s}(r))r+\frac{4  r}{3}}\right).
    \end{align*}
    Since $\mu \geq 24 (n+1) r \ln(n) (1+\ln_{2s}(r))$, we get
    \begin{align*}
        \Pr\left[\min_{k=1,...,T} p_{i,0}^{(k)} \leq \frac{1}{2r}\right] & \leq 2\exp\left(-\frac{24 (n+1) r \ln(n) (1+\ln_{2s}(r))}{12n (1+\ln_{2s}(r))r+\frac{4  r}{3}}\right) \\
                    & \leq 2\exp\left(-\frac{24 (n+1)  \ln(n) (1+\ln_{2s}(r))}{12(n+1)(1+\ln_{2s}(r))}\right) \\
                    & \leq 2\exp\left(-2 \ln(n)\right).
    \end{align*}
    Hence, it follows that
    \begin{gather*}
        \Pr\left[\min\nolimits_{k=1,...,T} p_{i,0}^{(k)} \leq \frac{1}{2r}\right]\leq \frac{2}{n^{2}}.
    \end{gather*}
    Applying a union bound over all~$n$ positions yields the result.
\end{proof}

In the proof of our next result, we apply the following Chernoff bound.
We apply it in order to quantify the number of iterations necessary to converge every position $i\in \left[n\right]$.
\begin{theorem}[Chernoff bound~{\cite[Theorem~$1.10.5$]{Doerr20bookchapter}}]
    \label{th:chernoff}
    Let $k \in \N_{\geq 1},\delta \in \left[0,1\right]$, and let $X$ be the sum of $k$ independent random variables each taking values in $\left[0,1\right]$. Then
    \begin{align*}
        \Pr\left[X\leq (1-\delta)\,\E\left[X\right]\right]\leq \exp\left( {-\frac{\delta^{2}\E\left[X\right]}{2}}\right).
    \end{align*}
\end{theorem}

An important concept for our analysis, following the approach by Doerr and Krejca~\cite{DoerrK21tcs}, is that a position is \emph{critical}.
Informally, a position is critical if and only if the frequencies corresponding to value~$0$ are for all smaller positions at the upper border.
Our runtime proof relies on showing that the \rUMDA quickly increases the frequency of a critical position to the upper border, thus making the next position critical.
Formally, let $t\in \N$.
We call a position $i\in \left[n\right]$ \emph{critical} for the \rUMDA on \rLO in iteration~$t$, if and only if for all $k \in [i - 1]$, it holds that $p_{k,0}^{(t)}=1-\frac{1}{n}$, and that $p_{i,0}^{(t)}<1-\frac{1}{n}$.

We now show that once a position $i\in \left[n\right]$ becomes critical, with high probability, with $s \in \R_{\geq 1}$ being an appropriate value separating~$\lambda$ from~$\mu$ (that is, defining the selection pressure), it takes less than $n \ln_{2s}(r+1)$ iterations to bring the frequency of the value $0$ to the upper border $1-\frac{1}{n}$.
We also prove that it remains there for a sufficient number of iterations until the convergence of the frequency matrix.
\begin{lemma}
    \label{critical_update}
    Let $s, u \in \R_{\geq 1}$.
     Consider the \rUMDA optimizing \rLO with $ \lambda \geq 3se\mu$ and $\mu \in \N_{\geq 1}$.
     Consider an iteration $t\in \N$ such that position $i\in \left[n\right] $ is critical, and let $b\in \R_{> 0}$ such that $p_{i,0}^{(t)} \geq b\geq \frac{2}{n}$.
     Then with a probability of at least $1 - u \ln_{2s}(\frac{1}{b}) \exp\bigl(-\frac{s \mu b}{24}\bigr)$, it holds for all $\theta\in \left[\ln_{2s}( \frac{1}{b}) .. u \ln_{2s}( \frac{1}{b})\right]$ that $p_{i,0}^{(t+\theta)} = 1-\frac{1}{n}$.
\end{lemma}
\begin{proof}
    We start by proving that, for all $\theta \in [0 .. u \ln_{2s}( \frac{1}{b})]$, the frequency $p^{(t + \theta)}_{i,0}$ multiplies by at least~$2s$ during an update, with high probability (and is then restricted).
    To this end, let~$t' \in [t .. t + \theta]$, and assume that $p_{i,0}^{(t')} \geq b$, and that position~$i$ or a position greater than~$i$ is critical (where we assume, for convenience, that if all frequencies for value~$0$ are $1 - \frac{1}{n}$, then position~$n + 1$ is critical).
    Furthermore, let~$X$ denote the number of sampled individuals in iteration~$t'$ that have at least~$i$ leading~$0$s.
    Note that $p_{i,0}^{(t)} \geq b$ by assumption as well as that~$i$ is critical in iteration~$t$.
    We discuss later via induction why these assumptions also hold for iteration~$t'$.

    We consider the process of sampling a single individual.
    Since position at least~$i$ is critical, by definition, for all $k \in [i - 1]$, we have $p_{k,0}^{(t')}=1-\frac{1}{n}$.
    Hence, the probability that all these positions are sampled as~$0$ for this individual is $( 1-\frac{1}{n} )^{i-1}\geq ( 1-\frac{1}{n} )^{n-1}\geq \frac{1}{e}$.
    This yields $\E\left[X\right] \geq \frac{\lambda p_{i,0}^{(t')}}{e}$, and since $\lambda \geq 3se\mu$, this yields $\E\left[X\right]\geq3s\mu p_{i,0}^{(t')}$.

    By the Chernoff bound (\Cref{th:chernoff}) and by the assumption $p_{i,0}^{(t')} \geq b$, we get
    \begin{align*}
        \Pr\left[ X\leq \frac{5}{2}s\mu p_{i,0}^{(t')}\right]
        &\leq \Pr\left[ X\leq \frac{5}{6}\E\left[X\right]\right]
        \leq \exp \left( -\frac{\E\left[X\right]}{72} \right)\\
        &\leq \exp \biggl( -\frac{s \mu p_{i,0}^{(t')} }{24} \biggr)
        \leq \exp \left(-\frac{s \mu b}{24}\right).
    \end{align*}

    We consider $\overline{p}_{i,0}^{(t'+1)}$ as defined in \Cref{sec:framework}, which is the updated frequency before being restricted to $\bigl[\tfrac{1}{(r-1)n}, 1 - \tfrac{1}{n}\bigr]$.
    Since $\overline{p}_{i,0}^{(t'+1)}\geq \min ( \frac{X}{\mu},1)$ by the definition of the update of the \rUMDA, we have
    \begin{align*}
        \Pr \left[\overline{p}_{i,0}^{(t'+1)}\leq \min \left( \frac{5}{2}sp_{i,0}^{(t')}, 1 \right) \right]
        \leq \Pr \left[X\leq \frac{5}{2}s\mu p_{i,0}^{(t')} \right]
        \leq \exp \left(-\frac{s \mu b}{24}\right).
    \end{align*}

    In order to update~$p_{i,0}^{(t')}$, the frequency vector~$\overline{p}_{i}^{(t'+1)}$ is restricted to the interval $\bigl[\frac{1}{(r - 1) n}, 1 - \frac{1}{n}\bigr]$, which entails that the updated frequency~$p_{i, 0}^{(t'+1)}$ may reduce when compared to~$\overline{p}_{i, 0}^{(t'+1)}$.
    However, since the restriction adds at most the lower border (that is, $\frac{1}{(r - 1)n}$) to a frequency, \emph{any} restriction rule adds at most a probability mass of~$\frac{1}{n}$ to the frequency vector.
    We assume pessimistically that, in order for the frequencies to sum to~$1$, this mass is entirely subtracted from~$\overline{p}_{i, 0}^{(t'+1)}$ during the restriction (noting that this does not take place once $\overline{p}_{i, 0}^{(t'+1)} \geq 1 - \frac{1}{n}$, as this means that it is set to the upper border instead).
    Further, the assumption $p_{i,0}^{(t')} \geq b \geq \frac{2}{n}$ yields that $\frac{5}{2}sp_{i,0}^{(t')} - \frac{1}{n} \geq 2sp_{i,0}^{(t')}$.
    Hence, we get that
    \begin{align*}
        &\Pr \Bigl[p_{i,0}^{(t'+1)} < \min \Bigl( 2sp_{i,0}^{(t')}, 1-\frac{1}{n} \Bigr) \Bigr]\\
        &\leq \Pr \Bigl[p_{i,0}^{(t'+1)} < \min \Bigl( \frac{5}{2} sp_{i,0}^{(t')} - \frac{1}{n}, 1-\frac{1}{n} \Bigr) \Bigr]
        \leq \exp \left(-\frac{s \mu b}{24}\right).
    \end{align*}

    By induction on the iteration~$t'$ (starting at~$t$), it follows that, with an additional failure probability of at most $\exp\bigl(-\frac{s \mu b}{24}\bigr)$ per iteration, the assumptions that $p_{i,0}^{(t')} \geq b$ and that position at least~$i$ is critical are satisfied.

    Starting from iteration~$t$, a union bound over the next $u \ln_{2s}(\frac{1}{b})$ iterations yields that the frequency~$p_{i, 0}$ continues growing exponentially with a factor of~$2s$ for the next $u \ln_{2s}( \frac{1}{b})$ iterations with probability at least $1- u \ln_{2s}(\frac{1}{b}) \exp\bigl(-\frac{s \mu b}{24}\bigr)$.
    Since, by assumption, ~$p_{i,0}^{(t)} \geq b$, it reaches $1 - \frac{1}{n}$ after at most $\ln_{2s}( \frac{1}{b})$ iterations during that time, concluding the proof.
\end{proof}

We now prove our main result.
\begin{proof}[Proof of \Cref{thm:rUMDAonrLO}]
    Since \rLO weakly prefers $0$s at all positions $i\in \left[n\right]$, by \Cref{lem:frequencyDoesNotGetLow}, with a probability of at least $1-\frac{2}{n}$, for all $i \in [n]$, the frequency $p_{i,0}$ remains above $\frac{1}{2r}$ for the first $n (1+\ln_{2s}(r))$ iterations.

    For each position $i \in [n]$, we apply \Cref{critical_update} with $b=\frac{1}{2r}$ and $u = n$, noting that the assumption $b\geq \frac{2}{n}$ is satisfied, since we assume $n\geq 4r$.
    Hence, for each $i \in [n]$, with a probability of at least $1- \ln_{2s}(2r) n^{1 - 0.5n}$, after at most $\ln_{2s}(2r)$ iterations, the frequency~$p_{i,0}$ is set to $1 - \frac{1}{n}$ and remains there for at least $(n-1) \ln_{2s}(2r)$ iterations.
    Further, by a union bound over all~$n$ frequency vectors, the above holds for all frequency vectors, with probability at least $1- \ln_{2s}(2r) n^{2 - 0.5n}$.

    Combining everything, with probability at least $1 - \frac{2}{n} - \ln_{2s}(2r) n^{2 - 0.5n}$, it holds by induction on position~$i$ that once position~$i$ is critical, the frequency~$p_{i, 0}$ reaches $1 - \frac{1}{n}$ in at most $\ln_{2s}(2r)$ iterations and remains there until at least iteration $n \ln_{2s}(2r)$.
    Since position~$0$ is critical in iteration~$0$, it follows that the frequencies for value~$0$ are set, in increasing order of their position, to $1 - \frac{1}{n}$.
    After at most $n \ln_{2s}(2r)$ iterations, all such frequencies are at the upper border, which proves the first part of the claim.

    For the second part, note that once $p_{n, 0} = 1 - \frac{1}{n}$, the population of the \rUMDA in that iteration contains at least $(1 - \frac{1}{n})\mu$ times the optimum.
    Further, each iteration accounts for~$\lambda$ fitness function evaluations.
    This proves the second claim.
\end{proof}

\subsubsection{Lower Bound}
\label{sec:lowerBound}

As the upper bound (\Cref{thm:rUMDAonrLO}), the lower bound shows an almost linear dependency of the number of iterations until the optimum is sampled for the first time with respect to~$\lambda$ and~$n$, only adding a factor of order~$\ln(r)$.
The difference of~$\ln(r)$ to the upper bound stems from the bound on~$\mu$, which is larger by a factor of around~$\ln(r)$ in the upper bound.

\begin{theorem}
    \label{thm:rUMDALowerBound}
    Let $\delta \in (0, 1)$ be a constant.
    Consider the \rUMDA optimizing \rLO with $\lambda \geq \mu \geq \max\{24 (n + 1) r \ln(n), 6 \frac{1 + \delta}{\delta^2} \ln(n)\}$.
    Furthermore, let $d = \lceil\log_{2r/3}((1 + \delta) \frac{\lambda}{\mu})\rceil = \lceil\frac{\ln((1 + \delta) \lambda / \mu)}{\ln(2r/3)}\rceil$ and let $\xi = \lceil\log_{2r/3}(n^2 \lambda)\rceil + 1$.
    Then with probability at least $1 - 4n^{-1}$, the \rUMDA does not sample the optimum in iteration $\lfloor\frac{n - \xi}{d}\rfloor - 1$ or earlier.
    This corresponds to more than $\lambda \lfloor\frac{n - \xi}{d}\rfloor$ fitness function evaluations until the optimum is sampled for the first time.
\end{theorem}

Our proof of \Cref{thm:rUMDALowerBound} follows closely the proof for a lower bound on the runtime of the UMDA on \LO in the binary case by Doerr and Krejca~\cite[Theorem~$6$]{DoerrK21tcs}.
The proof mainly relies on the leftmost position in a population that never had at least~$\mu$ samples with a~$0$ so far.
This position increases each iteration with high probability by only about $\ln(\frac{\lambda}{\mu})/\ln(r) \eqqcolon d$.
Before this position is sufficiently close to~$n$, it is very unlikely that the \rUMDA samples the optimum of \rLO.
Hence, the runtime is with high probability in the order of~$\frac{n}{d}$.

To make this outline formal, we say that a position $i \in [n]$ is \emph{selection-relevant in iteration $t \in \N$} (for \rLO) if and only if the population of the \rUMDA optimizing \rLO has in iteration~$t$ at least~$\mu$ individuals with at least $i - 1$ leading~$0$s.
Note that multiple positions can be selection-relevant in the same iteration, and that position~$1$ is always selection-relevant.
Furthermore, for each iteration $t \in \N$, we say that position $i \in [n]$ is the \emph{maximum selection-relevant position} if and only if~$i$ is the largest value among all selection-relevant positions in iteration~$t$.

The following lemma shows that the frequency for value~$0$ in positions that were not yet selection-relevant remain close to their starting value of~$\frac{1}{r}$, as they are neutral up to that point.
\begin{lemma}
    \label{lem:LOLowerBoundNeutralFrequencies}
    Let $b \in \N_{\geq 1}$.
    Consider the \rUMDA optimizing \rLO with $\lambda \geq \mu \geq 24 (b + 1) r \ln(b)$.
    For all $i \in [n]$, let~$T_i$ denote the first iteration such that position~$i$ is selection-relevant, and let $\Tsel_i = \min\{T_i, b\}$.
    Then with probability at least $1 - 2 n b^{-2}$, it holds for each $i \in [n]$ and each $t \in [0 .. \Tsel_i]$ that $p^{(t)}_{i, 0} \in (\frac{1}{2}\frac{1}{r}, \frac{3}{2}\frac{1}{r})$.
\end{lemma}

\begin{proof}
    Let $i \in [n]$ .
    We show that the sequence $(p^{(t)}_{i, 0})_{t \in \N}$ remains in $(\frac{1}{2}\frac{1}{r}, \frac{3}{2}\frac{1}{r})$ as long as $t \leq \Tsel_i$ by aiming to apply \Cref{th:neutral_bound}.
    We then conclude the proof via a union bound of the failure probabilities (that is, the probabilities that a frequency does not remain in said interval) over all possible values for~$i$.

    Conditional on~$\Tsel_i$, since~$i$ only becomes selection-relevant the earliest in iteration~$\Tsel$, position~$i$ is neutral up to (including) iteration~$\Tsel$.
    That is, for all $t \in [0 .. \Tsel - 1]$, position~$i$ has no influence on the fitness of each individual in population~$P^{(t)}$ (and thus on the updated frequency $p^{(t + 1)}_{i, 0}$).
    Hence, by \Cref{th:neutral_bound}, by $\Tsel \leq b$, and by the lower bound on~$\mu$, we get that
    \begin{align*}
        &\Pr\left[\max\nolimits_{s \in [0 .. \Tsel_i]}\ \left|p_{i,0}^{(s)}-\frac{1}{r}\right| \geq \frac{1}{2r}\,\middle\vert\, \Tsel_i\right]
        \leq 2\exp\left(-\frac{\mu}{12 \Tsel r + (4/3) r}\right)\\
        &\leq 2\exp\left(-\frac{\mu}{12 (b + 1) r}\right)
        \leq 2\exp\left(-\frac{24 (b + 1) r \ln(b)}{12 (b + 1) r}\right)
        \leq 2 b^{-2}.
    \end{align*}
    By the law of total probability, this bound also holds independently of the outcome of~$\Tsel$.

    Taking the union bound of the above bound over all~$n$ values for~$i$ yields that the overall failure probability is at most $2 n b^{-2}$, concluding the proof.
\end{proof}

For the next lemma, we make use of the following Chernoff bound, which we apply in order to show that new offspring does not extend the prefix of leading~$0$s by too much.
It is a non-trivial extension of the typical Chernoff bound to the case where we have an upper bound on the expected value of the sum of independent Bernoulli random variables.
This extension is non-trivial as the upper bound on the expectation also results in a stronger probability bound.
\begin{theorem}[Chernoff bound~{\cite[Theorem~$1.10.21$~(a) with Theorem~$1.10.1$]{Doerr20bookchapter}}]
    \label{thm:chernoffGreaterThanEX}
    Let $k \in \N_{\geq 1}$, and let $X$ be the sum of~$k$ independent random variables each taking values in $\left[0,1\right]$.
    Moreover, let $\delta, \mu^+ \in \R_{\geq 0}$ such that $\mu^+ \geq \E[X]$.
    Then
    \begin{align*}
        \Pr\left[X \geq (1 + \delta)\mu^+\right]\leq \exp\left(-\frac{1}{3} \min\left\{\delta^2, \delta\right\} \mu^+\right).
    \end{align*}
\end{theorem}

In the following lemma, we show that the maximum selection-relevant position increases each iteration with high probability by roughly $\log_r(\frac{\lambda}{\mu})$.
To this end, we tie it to the concept of a critical position, as defined in \Cref{sec:upperBound}.
This proof is heavily inspired by the proof of Doerr and Krejca~\cite[Lemma~$4$]{DoerrK21tcs}, but we fix a mistake in their proof, where the penultimate estimate of the application of the Chernoff bound bounds the exponent in the wrong direction.

\begin{lemma}
    \label{lem:selectionRelevantPositionNotMovingQuickly}
    Let $\delta \in (0, 1)$ be a constant.
    Consider the \rUMDA optimizing \rLO with $\mu \geq 6 \frac{1 + \delta}{\delta^2} \ln(n)$.
    Furthermore, consider an iteration $t \in \N$ such that position $i \in [n]$ is critical and that, for all positions $i' \in [i + 1 .. n]$, it holds that $p^{(t)}_{i', 0} \leq \frac{3}{2} \frac{1}{r}$.
    Let $d = \bigl\lceil\log_{2r/3}\bigl((1 + \delta) \frac{\lambda}{\mu}\bigr)\bigr\rceil$.
    Then, with probability at least $1 - n^{-2}$, the maximum selection-relevant position in iteration~$t$ is at most $\min\{n, i + d\}$.
\end{lemma}

\begin{proof}
    We note that $\lambda \geq \mu$ by the definition of the \rUMDA and since $\delta > 0$, it holds that $d \geq 1$.
    Furthermore, we assume that $i < n - d$, that is, it holds that $\min\{n, i + d\} = i + d$.
    For $i \geq n - d$, we statement claims that the maximum selection-relevant position is at most~$n$, which is trivially the case, as all positions are in~$[n]$.

    For a position $k \in [n]$ to become the maximum selection-relevant position in iteration~$t$, by definition, it is necessary that at least~$\mu$ individuals in population~$P^{(t)}$ have at least $k - 1$ leading~$0$s.
    We show via \Cref{thm:chernoffGreaterThanEX} that it is very unlikely that such a prefix of leading~$0$s extends by much.

    To this end, let $k = i + d$, and let~$X$ denote the number of individuals from~$P^{(t)}$ with at least~$k$ leading~$0$s.
    Since we assume that each frequency of value~$0$ at a position larger than~$i$ is at most~$\frac{3}{2} \frac{1}{r}$, as well as due to the independent sampling of the \rUMDA and due to the definition of~$d$, it follows that
    \begin{align*}
        \E[X] \leq \lambda \left(\frac{3}{2} \frac{1}{r}\right)^{d}
        = \lambda \left(\frac{2}{3} r\right)^{-d}
        \leq \lambda \frac{\mu}{(1 + \delta) \lambda}
        = \frac{\mu}{1 + \delta}.
    \end{align*}
    Hence, by applying \Cref{thm:chernoffGreaterThanEX} with $\mu^+ = \frac{\mu}{1 + \delta}$, recalling that $\delta \in  (0, 1)$, and by applying the bound on~$\mu$, we get that
    \begin{align*}
        \Pr[X \geq \mu] &= \Pr\left[X \geq (1 + \delta) \frac{\mu}{1 + \delta}\right]
        \leq \exp\left(-\frac{1}{3}\min\left\{\delta^2, \delta\right\} \frac{\mu}{1 + \delta}\right)\\
        &= \exp\left(-\frac{1}{3} \mu\frac{\delta^2}{1 + \delta}\right)
        \leq n^{-2}.
    \end{align*}
    Consequently, with probability at least $1 - n^{-2}$, the population~$P^{(t)}$ contains fewer than~$\mu$ offspring that have at least~$k$ leading~$0$s.
    That is, the largest position $k' \in [n]$ where at least~$\mu$ offspring have at least~$k'$ leading~$0$s is at most $k - 1$, which is equivalent to the maximum selection-relevant position being at most~$k$.
\end{proof}

The next lemma is the last one before we prove our lower bound.
The lemma shows that it is very unlikely for the \rUMDA to sample the optimum of \LO while many frequencies for value~$0$ are not high yet (which is measured by the critical position).

\begin{lemma}
    \label{lem:samplingLOOptimumIsUnlikely}
    Consider the \rUMDA optimizing \rLO, and consider an iteration $t \in \N$ and a position $i \in [n]$ such that, for all positions $i' \in [i + 1 .. n]$, it holds that $p^{(t)}_{i', 0} \leq \frac{3}{2} \frac{1}{r}$.
    Then, with probability at least $1 - \lambda (\frac{3}{2} \frac{1}{r})^{n - i}$, the \rUMDA does not sample the optimum in this iteration.
\end{lemma}

\begin{proof}
    We bound the probability for sampling the optimum this iteration from above.
    The probability for a single offspring to be the optimum is, due to the upper bound on the last $n - i$ frequencies, at most $(\frac{3}{2} \frac{1}{r})^{n - i}$, as all positions need to be a~$0$.
    Taking a union bound over all~$\lambda$ samples of this iteration concludes the proof.
\end{proof}

\Cref{lem:frequencyDoesNotGetLow,lem:selectionRelevantPositionNotMovingQuickly,lem:samplingLOOptimumIsUnlikely} are sufficient for proving \Cref{thm:rUMDALowerBound}.

\begin{proof}[Proof of \Cref{thm:rUMDALowerBound}]
    We only show the bound on the number of iterations.
    Since we start counting iterations at~$0$ and since the \rUMDA creates exactly~$\lambda$ offspring each iteration, the bound on the number of fitness function evaluations follows immediately.

    For the entirety of the proof, we assume that during the first~$n$ iterations, all frequencies for value~$0$ remain in $(\frac{1}{2}\frac{1}{r}, \frac{3}{2}\frac{1}{r})$ as long as they did not become selection-relevant yet.
    By \Cref{lem:LOLowerBoundNeutralFrequencies} with $b = n$, noting that~$\mu$ is sufficiently large, this occurs with probability at least $1 - 2n^{-1}$.
    Furthermore, we assume that $n - \xi \geq d$, as \Cref{thm:rUMDALowerBound} yields a trivial lower bound of~$0$ otherwise.

    We continue by proving via induction on $t \in [0 .. n]$ that with probability at least $1 - (t + 1)n^{-2}$ it holds that each position $i \in [(t + 1)d + 2 .. n]$ is not relevant up to (including) iteration~$t$.

    For the base case $t = 0$, by the definition of the \rUMDA, for all positions $i \in [n]$, it holds that $p^{(0)}_{i, 0} = \frac{1}{r}$.
    This especially means that position~$0$ is critical this iteration.
    Applying \Cref{lem:selectionRelevantPositionNotMovingQuickly}, noting that the requirements for~$\delta$ and~$\mu$ are met, proves the base case, as, with probability at least $1 - n^{-2}$, the maximum selection-relevant position in iteration~$0$ is~$d$.

    For the inductive step, assume that the inductive hypothesis holds up to (including) iteration $t \in [0 .. n - 1]$.
    Hence, with probability at least $1 - (t + 1)n^{-2}$, the maximum selection relevant-position in iteration~$t$ (and up to there) is at most $(t + 1)d + 1$.
    This implies that the critical position $k \in [n]$ in iteration $t + 1$ is also at most $(t + 1)d + 1$.
    Furthermore, all frequencies for value~$0$ at positions greater than $(t + 1)d + 1$ have not been selection-relevant yet.
    Thus, by our argument at the beginning of the proof, these frequencies are at most~$\frac{3}{2} \frac{1}{r}$.
    Overall, by \Cref{lem:selectionRelevantPositionNotMovingQuickly}, in iteration $t + 1$, with probability at most~$n^{-2}$, the maximum selection-relevant position in iteration $t + 1$ is at least $k + d + 1$.
    Via a union bound with the failure probability of the inductive hypothesis, this proves the claim, that is, with probability at least $1 - (t + 2)n^{-2}$, the maximum-selection relevant position in iteration $t + 1$ is at most $k + d \leq (t + 2)d + 1$.

    This claim shows that, for $t' = \lfloor\frac{n - \xi}{d}\rfloor - 1 \leq n$, with probability at least $1 - n^{-1}$, each position greater than $n - \xi + 1$ is never selection-relevant up to (including) iteration~$t'$.
    Hence, by our argument at the beginning of the proof, these frequencies are at most~$\frac{3}{2} \frac{1}{r}$.
    Applying \Cref{lem:samplingLOOptimumIsUnlikely} with $i = n - \xi + 1$ then yields that the \rUMDA does not sample the optimum in each iteration up to~$t'$ with a probability of at least $1 - \lambda (\frac{3}{2} \frac{1}{r})^{n - i} = 1 - \lambda (\frac{3}{2} \frac{1}{r})^{\xi - 1} \geq 1 - n^{-2}$ per iteration.
    A union bound over at most $t' + 1 \leq n$ iterations then shows that with probability at least $1 - n^{-1}$, it holds that up to (including) iteration~$t'$, the \rUMDA does not sample the optimum.

    Last, a union bound over the three error probabilities of the three arguments above then shows that with probability at least $1 - 4n^{-1}$, the \rUMDA does not sample the optimum up to (including) iteration~$t'$, concluding the proof.
\end{proof}

\section{Conclusion}

We have proposed the first systematic framework of EDAs for problems with multi-valued decision variables. Our analysis of the genetic-drift effect and our runtime analysis on the multi-valued version of \leadingones have shown that the increase in decision values does not result in significant difficulties. Although there may be a slightly stronger genetic drift (requiring a more conservative model update, that is, a higher selection size $\mu$ for the UMDA) and slightly longer runtimes, these outcomes are to be expected given the increased complexity of the problem. We hope that our findings will inspire researchers and practitioners to embrace the benefits of EDAs for multi-valued decision problems, beyond the previously limited application to mostly permutations and binary decision variables.

\section*{Acknowledgments}

Thank you to Josu Ceberio for some useful discussions. This work also profited from many scientific discussions at the Dagstuhl Seminar 22182 ``Estimation-of-Distribution Algorithms: Theory and Applications''.
This work was supported by a public grant as part of the
Investissements d'avenir project, reference ANR-11-LABX-0056-LMH, LabEx LMH.

\bibliographystyle{plain}
\bibliography{alles_ea_master,ich_master,rest}
}

\end{document}